\documentclass[12pt]{colt2018}

\usepackage{microtype} 
\usepackage{algpseudocode}
\usepackage{parskip}

\usepackage{mathtools}
\usepackage{amsmath}
\usepackage{bbm}
\usepackage{amsfonts}
\usepackage{amssymb}
\usepackage{wrapfig}
\usepackage{graphicx}
\usepackage{latexsym}
\usepackage[mathscr]{euscript}
\usepackage{times}
\usepackage{fullpage}
\let\vec\undefined
\usepackage{MnSymbol}
\usepackage{xpatch}
\usepackage{float}
\def\Rset{\mathbb{R}}
\usepackage{natbib}

\newcommand{\disc}{\text{disc}}
\newcommand{\indic}{1}
\newcommand{\tcO}{\tilde{\mathcal{O}}}

\DeclareMathOperator*{\E}{\mathbb E}
\DeclareMathOperator*{\argmax}{argmax}
\DeclareMathOperator*{\argmin}{argmin}

\newcommand{\conf}[1]{}
\newcommand{\arxiv}[1]{#1}

\newtheorem*{rep@theorem}{\rep@title}
\newcommand{\newreptheorem}[2]{%
\newenvironment{rep#1}[1]{%
 \def\rep@title{#2 \ref{##1}}%
 \begin{rep@theorem}}%
 {\end{rep@theorem}}}

\newreptheorem{theorem}{Theorem}
\newreptheorem{lemma}{Lemma}
\newreptheorem{corollary}{Corollary}
\newreptheorem{proposition}{Proposition}
\newreptheorem{example}{Example}
\newreptheorem{properties}{Properties}

\newcommand{\cH}{\mathcal{H}}

\newcommand{\cO}{\mathcal{O}}

\newcommand{\cX}{\mathcal{X}}
\newcommand{\cY}{\mathcal{Y}}

\newcommand{\sC}{{\mathscr C}}
\newcommand{\sD}{{\mathscr D}}

\newcommand{\sH}{{\mathscr H}}

\newcommand{\sL}{{\mathscr L}}

\newcommand{\sU}{{\mathscr U}}

\newcommand{\sX}{{\mathscr X}}
\newcommand{\sY}{{\mathscr Y}}

\newcommand{\bm}{{\mathbf m}}

\newcommand{\R}{{\mathfrak R}}
\newcommand{\s}{{\mathfrak s}}

\newcommand{\h}{\widehat}

\newcommand{\ignore}[1]{}

\hypersetup{
  colorlinks   = true,
  urlcolor     = blue,
  linkcolor    = blue,
  citecolor   = blue
}

\title{Three Approaches for Personalization with Applications to Federated  Learning}

\coltauthor{
\AND
\Name{Yishay Mansour}\Email{mansour@google.com}\\
\addr Google Research and Tel Aviv University
\AND
\Name{Mehryar Mohri}\Email{mohri@google.com}\\
\addr Google Research and 
Courant Institute of Mathematical Sciences, New York
\AND
\Name{Jae Ro}\Email{jaero@google.com}\\
\addr Google Research, New York
\AND
\Name{Ananda Theertha Suresh}\Email{theertha@google.com}\\
\addr Google Research, New York
}
\begin{document}

\maketitle

\begin{abstract}
The standard objective in machine learning is to train a single model for all users. However, in many learning scenarios, such as cloud computing and federated learning, it is possible to learn a personalized model per user. In this work, we present a systematic learning-theoretic study of personalization. We propose and analyze three approaches: user clustering, data interpolation, and model interpolation. For all three approaches, we provide learning-theoretic guarantees and efficient algorithms for which we also demonstrate the performance empirically. All of our algorithms are model-agnostic and work for any hypothesis class.
\end{abstract}

\section{Introduction}

A popular application of language models is virtual keyboard
applications, where the goal is to predict the next word, given the
previous words~\citep{hard2018federated}. For example, given ``I live in the state of", ideally, it should guess the state the
user intended to type. However, suppose we train a single model on all
the user data and deploy it, then the model would predict the same
state for all users and would not be a good model for most.  Similarly, in many practical applications, the distribution of
data across clients is highly non-i.i.d.\ and training a
single global model for all clients may not be optimal.

 Thus, we study the problem of learning \emph{personalized models},
 where the goal is to train a model for each client, based on the
 client's own dataset and the datasets of other clients. Such an
 approach would be useful in applications with the natural
 infrastructure to deploy a personalized model for each client, which
 is the case with large-scale learning scenarios such as
 \emph{federated learning} (FL)~\citep{mcmahan2017communication}.
 
 Before we proceed further, we highlight one of our use cases in FL.
 In FL, typically a centralized global model is trained based on data from a large number of clients, which may be mobile
 phones, other mobile devices, or sensors
 \citep{konevcny2016federated,konecny2016federated2,mcmahan2017communication,
 yang2019federated} using a variant of
 stochastic gradient descent called \emph{FedAvg}.  This global model
 benefits from having access to client data and can often
 perform better on several learning
 problems, including next word prediction \citep{hard2018federated,
   yang2018applied} and predictive models in health
 \citep{brisimi2018federated}. We refer to
 Appendix~\ref{app:federated} for more details on FL.

Personalization of machine learning models has been studied extensively for specific applications such as speech recognition~\citep{yu2017recent}. However, many algorithms are speech specific or not suitable for FL due to distributed constraints. Personalization is also related to Hierarchical Bayesian models~\citep{gelman2006multilevel, allenby2005hierarchical}. However, they are not directly applicable for FL. Personalization in the context of FL has been studied by several works via multi-task learning \citep{smith2017federated}, meta-learning~\citep{jiang2019improving, khodak2019adaptive}, use of local parameters~\citep{arivazhagan2019federated, liang2020think}, mixture of experts~\citep{peterson2019private}, finetuning and variants~\citep{wang2019federated, yu2020salvaging} among others. We refer readers to Appendix~\ref{sec:related} for an overview of works on personalization in FL.
 
We provide a learning-theoretic framework, generalization guarantees,
and computationally efficient algorithms for personalization. Since
FL is one of the main frameworks where personalized models can be used, we propose efficient algorithms that take into account
computation and communication bottlenecks.
\section{Preliminaries}
\label{sec:when}

Before describing the mathematical details of personalization, we
highlight two related models. The first one is the global model trained on data from all the clients. This can be trained
using either standard empirical risk
minimization~\citep{vapnik1992principles} or other methods such as
agnostic risk minimization \citep{mohri2019agnostic}.  The second
baseline model is the purely local model trained only on the
client's data.

The global model is trained on large amounts of data and
generalizes well on unseen test data; however it does not perform well for clients whose
data distributions are very different from the global train data distribution. On
the other hand, the train data distributions of local models match the ones at inference time, but they do not generalize well due to the scarcity of data.

Personalized models can be viewed as intermediate models between pure-local
and global models. Thus, the hope is that they incorporate the
generalization properties of the global model and the distribution matching
property of the local model. Before we proceed further,
we first introduce the notation used in the rest of the paper.

\subsection{Notation}

We start with some general notation and definitions used throughout
the paper. Let $\sX$ denote the input space and $\sY$ the output
space. We will primarily discuss a multi-class classification problem
where $\sY$ is a finite set of classes, but much of our results can be
extended straightforwardly to regression and other problems.  The
hypotheses we consider are of the form $h\colon \sX \to \Delta^\sY$,
where $\Delta^\sY$ stands for the simplex over $\sY$. Thus, $h(x)$ is
a probability distribution over the classes that can be
assigned to $x \in \sX$. We will denote by $\sH$ a family of such
hypotheses $h$.  We also denote by $\ell$ a loss function defined over
$\Delta^\sY \times \sY$ and taking non-negative values. The loss of $h
\in \sH$ for a labeled sample $(x, y) \in \sX \times \sY$ is given by
$\ell(h(x), y)$. Without loss of generality, we assume
  that the loss $\ell$ is bounded by one. 
  We will denote by $\sL_\sD(h)$ the expected loss of a
  hypothesis $h$ with respect to a distribution $\sD$ over $\sX \times
\sY$:
\vspace{-0.5ex}
\[
\sL_\sD(h) = \E_{(x, y) \sim \sD} [\ell(h(x), y)],
\]
and by $h_\sD$ its minimizer: $h_\sD = \argmin_{h \in \sH}
\sL_\sD(h)$. Let $\R_{{\sD}, m}(\sH)$ denote the Rademacher
complexity of class $\sH$ over the distribution ${\sD}$ with $m$
samples.

Let $p$ be the number of clients.  The distribution of samples of
client $k$ is denoted by $\sD_k$. Clients do not know the true
distribution, but instead, have access to $m_k$ samples drawn i.i.d.\ 
from the distribution $\sD_k$. We will denote by $\h{\sD}_k$ the
corresponding empirical distribution of samples and by $m = \sum^p_{k=1} m_k$
the total number of samples.

\subsection{Local model}
\label{sec:local}
We first ask when it is beneficial for a client to
participate in global model training. Consider a canonical user with distribution
$\sD_1$.  Suppose we train a purely local model based on the client's
data and obtain a model $h_{\h{\sD}_1}$.  By standard learning-theoretic
tools~\citep{mohri2018foundations}, the performance of this
model can be bounded as follows: with probability at least $1- \delta$,
the minimizer of empirical risk $\sL_{\h{\sD}_1}(h)$ satisfies
\vspace{-1ex}
\begin{equation}
    \label{eq:local}
\sL_{\sD_1}(h_{\h{\sD}_1}) - \sL_{\sD_1}(h_{\sD_1})=  \cO \left(\R_{\sD_1, m_1}(\sH)\right) = 
 \cO
\left(\frac{\sqrt{d + \log {1}/{\delta}}}{\sqrt{m_1}}
\right),
\end{equation}
where $\R_{\sD_1, m_1}(\sH)$ is the Rademacher complexity and 
 $d$ is the pseudo-dimension of the hypothesis class $\sH$~\citep{mohri2018foundations}. Note that
pseudo-dimension coincides with VC dimension for $0-1$ loss.
From~\eqref{eq:local}, it is
clear that local models perform well when the number of samples $m_1$
is large. However, this is often not the case. In many realistic
settings, such as virtual keyboard models, the average number of
samples per user is in the order of hundreds, whereas the pseudo-dimension
of the hypothesis class is in millions~\citep{hard2018federated}. In
such cases, the above bound becomes vacuous.

\subsection{Uniform global model}
\label{sec:uniform_global}
The global model is trained by minimizing the empirical risk on the
concatenation of all the samples. For $\lambda \in \Delta^p$, the weighted average
distribution $\sD_\lambda$ is given by $\sum_{k} \lambda_k \sD_k$. The
global model is trained on the concatenated samples
from all the users and hence is equivalent to minimize the loss on the
distribution $\hat{\sU} = \sum_{k} \lambda'_k \h{\sD}_k$, where
$\lambda'_k = m_k/m$. Since the global model is trained on data
from all the clients, it may not match the actual underlying client
distribution and thus may perform worse.

The divergence between distributions is often measured by a Bregman
divergence such as KL-divergence or unnormalized relative entropy. 
However, such divergences do not consider the
underlying machine learning task at hand for example learning the best hypotheses out of $\sH$. To obtain better bounds, we use the notion of
label-discrepancy between distributions~\citep{mansour2009domain, mohri2012new}.  For
two distributions over features and labels, $\sD_1$ and $\sD_2$, and a
class of hypotheses $\cH$, label-discrepancy\citep{mohri2012new} is given by
\vspace{-0.5ex}
\[
\disc_{\cH}(\sD_1, \sD_2) = \max_{h\in \cH} | \sL_{\sD_1}(h) -
\sL_{\sD_2}(h)|.
\]
If the loss of all the hypotheses in the class is the same under both
$\sD_1$ and $\sD_2$, then the discrepancy is zero and models trained
on $\sD_1$ generalize well on $\sD_2$ and vice versa. 

With the above definitions, it can be shown that the uniform global
model generalizes as follows: with probability at least $1- \delta$, the
minimizer of empirical risk on the uniform distribution satisfies
\begin{equation}
    \label{eq:global}
    \sL_{\sD_1}(h_{\hat \sU}) - \sL_{\sD_1}(h_{\sD_1})
   = \cO \left(\R_{\sU, m}(\sH)\right) + 
\disc_{\cH}(\sD_1, \sU) =
 \cO \left( \frac{\sqrt{d + \log {1}/{\delta}}}{\sqrt{m}} \right) + 
\disc_{\cH}(\sD_1, \sU) .
\end{equation}
Since the global model is trained on the concatenation of all users' data,
it generalizes well. However, due to the distribution mismatch, the model
may not perform well for a specific user. If ${\sU} = \sum_{k} \lambda'_k {\sD}_k$,
the
difference between local and global models depends on the discrepancy
between  $\sD_1$ and $\sU$, $m_1$ the number of
samples from the domain $\sD_1$, and the total number of samples
$m$. While in most practical applications $m_1$ is small and
hence a global model usually performs better, this is not guaranteed. We provide a simple example illustrating such a case in Appendix~\ref{app:example}.

Since the uniform global model assigns weight $m_k/m$ to client $k$,
clients with larger numbers of samples receive higher importance. This can
adversely affect clients with small amounts of data. Furthermore,
by~\eqref{eq:global}, the model may not generalize well for
clients whose distribution is different than the uniform distribution.
Thus,~\eqref{eq:local} and~\eqref{eq:global} give some guidelines 
under which it is beneficial for clients to participate in global model training.

Instead of using uniform weighting of samples, one can use agnostic risk proposed by~\cite{mohri2019agnostic},
which is more risk averse. We refer to Appendix~\ref{app:agnostic} for details
about the agnostic risk minimization.

\section{Our contributions}
\label{sec:theory}

We ask if personalization can be achieved by an intermediate
model between the local and global models. Furthermore, for ease of applicability and to satisfy the communication constraints in FL, we focus on scalable algorithms with low communication bottleneck. This gives rise to three
natural algorithms, which are orthogonal and can be used separately or
together.
\begin{itemize}
    \item Train a model for subsets of users: we can cluster users
      into groups and train a model for each group. We refer to this
      as \emph{user clustering}, or more refinely
      \emph{hypothesis-based clustering}.
    \item Train a model on interpolated data: we can combine the
      local and global data and train a model on their combination. We
      refer to this as \emph{data interpolation}.
    \item Combine local and global models: we can train a local and a
      global model and use their combination. We refer to this as
      \emph{model interpolation}.
\end{itemize}

We provide generalization bounds and communication-efficient algorithms for all of the above methods. 
We show that the above three methods has small communication bottleneck and enjoys qualitative privacy benefits similar to training a global model.
Of the three proposed approaches, data interpolation has non-trivial communication cost and data security.
We show that data interpolation can be implemented with small communication overhead in Section~\ref{sec:dapper}. We also show discuss data security aspect and methods to improve it in Appendix~\ref{sec:practical}.

Of the remaining methods, model interpolation has the same communication cost and security as that of training a single model. Clustering has the same data security as that of training single models, but the communication cost is $q$ times that of training a single model, where $q$ is the number of clusters. In the rest of the paper, we study each of the above methods.

\section{User clustering}

Instead of training a single global model, a natural approach is to cluster clients into groups and train a model for each
group.  This is an intermediate model between a purely local and global model and provides a trade-off between generalization and
distribution mismatch.
If we have a clustering of users, then we can naturally find a model
for each user using standard optimization techniques. In this section,
we ask how to define clusters. Clustering is a classical problem with
a broad literature and known algorithms~\citep{jain2010data}.  We
argue that, since the subsequent application of our clustering is
known, incorporating it into the clustering algorithm will be
beneficial. We refer readers to  Appendix~\ref{app:cluster_baseline} for more details on comparison to baseline works.

\subsection{Hypothesis-based clustering}
 Consider the scenario where we are interested in finding
 clusters of images for a facial recognition task. Suppose we are
 interested in finding clusters of users for each gender and find
 a good model for each cluster. If we naively use the Bregman
 divergence clustering, it may focus on clustering based on the image
 background e.g., outdoor or indoors to find clusters instead of gender.

To overcome this, we propose to incorporate the
task at hand to obtain better clusters. We refer to this approach as
\emph{hypothesis-based clustering} and show that it admits better
generalization bounds than the Bregman divergence approach. We partition users into $q$ clusters and find the best hypothesis for
each cluster. In particular, we use the following optimization:
\begin{equation}
    \label{eq:cluster}
    \min_{h_1, \ldots, h_q} \sum^p_{k=1} \lambda_k \cdot \min_{i
      \in [q]} \sL_{\sD_k}(h_i),
\end{equation}
where $\lambda_k$ is the importance of client $k$.  The above loss
function trains $q$ best hypotheses and naturally divides $\cX \times
\cY$ into $q$ partitions, where each partition is associated with a
particular hypothesis $h_k$. In practice, we only have access to the empirical distributions
$\h{\sD}_k$.  We replace $\sL_{\sD_k}(h_i)$ by $\sL_{\h \sD_k}(h_i)$ in optimization. To simplify the analysis, we use the fraction of samples
from each user $m_k/m$ as $\lambda_k$. An alternative
approach is to use $\lambda_k = 1/p$ for all users, which assigns
equal weight to all clients. The analysis is similar and
we omit it to be concise.

\subsection{Generalization bounds}
We now analyze the generalization properties of this technique. We bound the maximum difference between true cluster based loss and empirical cluster based loss for all hypotheses. We note that such a generalization bound holds for any clustering algorithm (see Appendix~\ref{app:lem_cluster}).

 Let $C_1,C_2\ldots,C_q$ be
 the clusters and let $m_{C_i}$ be the number of samples from cluster
 $i$. Let $\sC_i$ and $\h{\sC}_i$ be the empirical and true
 distributions of cluster $\sC_i$. With these definitions, we now bound
 the generalization error of this technique.
 
\begin{theorem}[Appendix~\ref{app:cluster}]
\label{thm:cluster}
With probability at least $1-\delta$,
\begin{align*}
& \max_{h_1, \ldots, h_q} \left \lvert  \sum^p_{k=1} \frac{m_k}{m} \cdot \left( \min_{i \in [q]} \sL_{\sD_k}(h_i)  - 
\min_{i \in [q]} \sL_{\h{\sD}_k}(h_i) \right)\right \rvert \leq 
 2 \sqrt{\frac{p \log \frac{2q}{\delta}}{m}}
 + \max_{C_1, \ldots, C_q} \sum^q_{i=1} \frac{m_{C_i}}{m} \R_{\sC_i, m_{C_i}}(\sH).
\end{align*}
\end{theorem}
The above result implies the following corollary, which is easier to interpret.
\begin{corollary}[Appendix~\ref{app:cor_cluster}]
\label{cor:cluster}
Let $d$ be the pseudo-dimension of $\sH$. Then with probability at least
$1-\delta$, the following holds:
\begin{align*}
    \conf{&} \max_{h_1, \ldots, h_q} \left \lvert \sum^p_{k=1}
    \frac{m_k}{m} \cdot \left(\min_{i \in [q]} \sL_{\sD_k}(h_i) -
    \min_{i \in [q]}
    \sL_{\h{\sD}_k}(h_i)\right) \right \rvert 
    \leq 
    \sqrt{\frac{4p \log \frac{2q}{\delta}}{m}} + \sqrt{\frac{dq}{m}\log
      \frac{em}{d}}.
\end{align*}
\end{corollary}
The above learning bound can be understood as follows. For good
generalization, the average number of samples per user $m/p$ should be
larger than the logarithm of the number of clusters, and the average
number of samples per cluster $m/q$ should be larger than the
pseudo-dimension of the overall model. Somewhat surprisingly, these
results do not depend on the minimum number of samples per clients and instead depend only on the average statistics.

To make a comparison between the local performance~\eqref{eq:local} and 
the global model 
performance~\eqref{eq:global}, observe that 
combining~\eqref{eq:2} and  Corollary~\ref{cor:cluster} together with the 
definition of discrepancy yields
\begin{align*}
        \conf{&} 
\sum^p_{k=1} \frac{m_k}{m} \cdot\left( \sL_{\sD_k}(\h{h}_{f(k)})  
-  \min_{h \in H} \sL_{\sD_k}(h) \right)
      \conf{\\ &} 
      \leq 
 2 \sqrt{\frac{p \log \frac{2q}{\delta}}{m}}
 + \sqrt{\frac{dqe}{m}\log \frac{m}{d}} + \sum^p_{k=1} \frac{m_k}{m} \disc(\sD_k, \sC_{f(k)}),
\end{align*}
where $f \colon [p] \to [q]$ is the mapping from users to clusters.
  Thus, the generalization bound is in between that of the local and
  global model.  For $q=1$, it yields the global model, and for $q=p$, it
  yields the local model.  As we increase $q$, the generalization
  decreases and the discrepancy term gets smaller. Allowing a general
  $q$ lets us choose the best clustering scheme and provides a smooth
  trade-off between the generalization and the distribution
  matching. In practice, we choose small values of $q > 1$. We further
  note that we are not restricted to using the same value of $q$ for
  all clients. We can find clusters for several values of $q$ and use
  the best one for each client separately using a hold-out set of samples.

\subsection{Algorithm : \textsc{HypCluster}}

\begin{figure}[t]
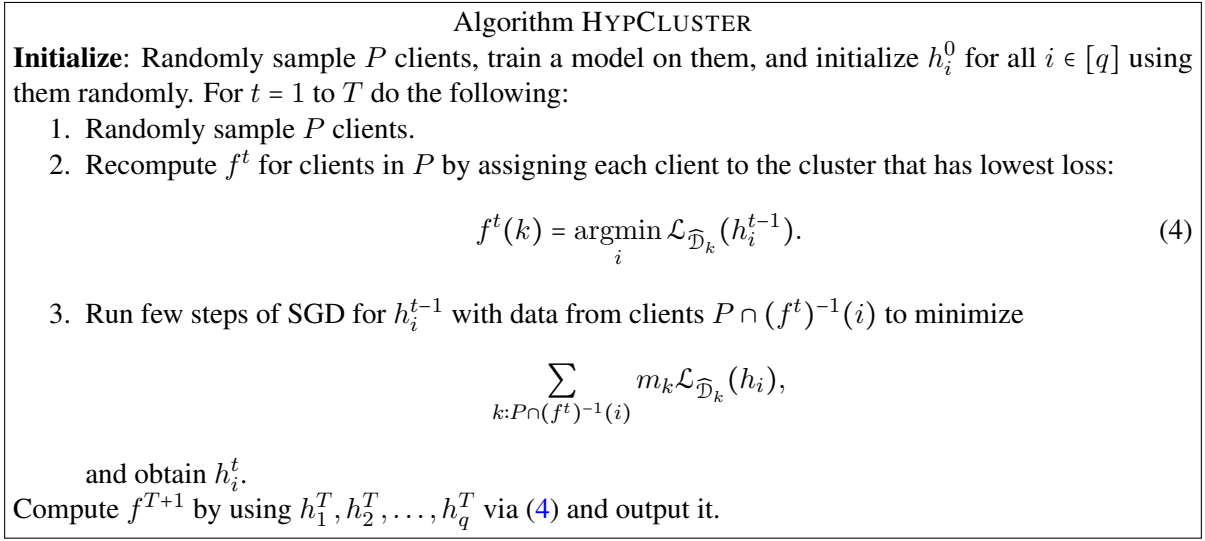

\begin{center}
\fbox{\begin{minipage}{\conf{0.45}\arxiv{0.95}\textwidth}
\begin{center}
Algorithm \textsc{HypCluster}
\end{center}
\textbf{Initialize}: Randomly sample $P$ clients, train a model on
them, and initialize $h^0_i$ for all $i \in [q]$ using them randomly.
For $t = 1$ to $T$ do the following:
\begin{enumerate}
    \item Randomly sample $P$ clients.
        \item Recompute $f^t$ for clients in $P$ by assigning each
          client to the cluster that has lowest loss:
    \begin{equation}
\label{eq:maximization}
    f^t(k) = \argmin_{i} \sL_{\h{\sD}_k}(h^{t-1}_i).
    \end{equation}
    \item Run few steps of SGD for $h^{t-1}_i$ with data from
      clients $P \cap (f^{t})^{-1}(i)$ to minimize
    \[
 \sum_{k : P \cap (f^{t})^{-1}(i)} m_k \sL_{\h{\sD}_k}(h_i),
    \]
    and obtain $h^t_i$.
\end{enumerate}
Compute $f^{T+1}$ by using $h^T_1, h^T_2,\ldots, h^T_{q}$
via~\eqref{eq:maximization} and output it.
\end{minipage}}
\end{center}
\caption{Pseudocode for \textsc{HypCluster} algorithm.}
\end{figure}

We provide an \emph{expectation-maximization} (EM)-type algorithm for
finding clusters and hypotheses. A naive EM modification may require
heavy computation and communication resources. To overcome this, we
propose a stochastic EM algorithm in \textsc{HypCluster}.
In the algorithm, we denote clusters via a mapping $f \colon [p] \to
[q]$, where $f(k)$ denotes the cluster of client $k$. Similar to 
$k$-means,
\textsc{HypCluster} is not guaranteed to converge to the true optimum,
but, as stated in the beginning of the previous section, the
generalization guarantee of Theorem~\ref{thm:cluster} still holds
here.

\section{Data interpolation}
\label{sec:dapper}
From the point of view of client $k$, there is a small amount of data
with distribution ${\sD_k}$ and a large amount of data from the global
or clustered distribution ${\sC}$. How are we to use auxiliary data from
$\sC$ to improve the model accuracy on $\sD_k$?  This relates the
problem of personalization to domain adaptation.  In domain
adaptation, there is a single source distribution, which is the global
data or the cluster data, and a single target distribution, which is the local client data. As in domain adaptation with target
labels~\citep{blitzer2008learning}, we have at our disposal a large
amount of labeled data from the source (global data) and a small
amount of labeled data from the target (personal data). We propose to
minimize the loss on the concatenated data,
\begin{equation}
    \label{eq:data}
\lambda \cdot \sD_k + (1-\lambda) \cdot \sC,
\end{equation}
where $\lambda$ is a hyper-parameter and can be obtained by either 
cross validation or by using the generalization bounds
of~\citet{blitzer2008learning}. $\sC$ can either be the
uniform distribution $\sU$ or one of the distributions obtained via
clustering.

Personalization is different from most domain adaptation works as they
assume they only have access to unlabeled target data~\citep{MansourMohriRostamizadeh2009a,ganin2016domain, zhao2018adversarial}, whereas in personalization we have
access to labeled target data. Secondly, we have one target domain per
client, which makes our problem computationally expensive, which we
discuss next. Given the known learning-theoretic bounds, a natural
question is if we can efficiently estimate the best hypothesis for a
given $\lambda$.  However, note that naive approaches suffer from the
following drawbacks.  If we optimize for each client separately, the
time complexity of learning per client is $\cO(m)$ and the overall
time complexity is $\cO(m \cdot p)$.

In addition to the computation time, the algorithm also
admits a high communication cost in FL. This is
because, to train the model with a $\lambda$-weighted mixture requires
the client to admit access to the entire dataset $\h{\sC}$, which
incurs communication cost $\cO(m)$. One empirically popular approach
to overcome this is the fine-tuning approach, where the central model
is fine-tuned on the local data~\citep{wang2019federated}.  However,
to the best of our knowledge, there are no theoretical guarantees and
the algorithm may be prone to catastrophic
forgetting~\citep{goodfellow2013empirical}. In fine-tuning, the models are typically trained first on the global data and then on the client's local data. Hence, the order in which samples are seen are not random. Furthermore, we only care about the models' performance on the local data. Hence, one cannot directly use known online-to-batch conversion results from online learning to obtain theoretical guarantees.

We propose \textsc{Dapper}, a theoretically motivated and efficient algorithm to
overcome the above issues. 
 The algorithm first trains a central model on the overall empirical
 distribution $\h{\sC}$.  Then for each client, it subsamples
 $\h{\sC}$ to create a smaller dataset of size $\h{\sC'}$ of size $r
 \cdot m_k$, where $r$ is a constant. It then minimizes the loss on weighted
 combination of two datasets i.e., $ \lambda \sL_{\h{\sD}_k}(h) + (1-\lambda) \sL_{\h{\sC'}}(h)$
 for several
 values of $\lambda$. Finally, it chooses the best $\lambda$ using
 cross-validation.
 The
 algorithm is efficient both in terms of its communication complexity
 which is $r \cdot m_k$ and its computation time, which is at most $(r+1)\cdot
 m_k$. Hence, the overall communication and computation time is $\cO(r
 \cdot m)$. Due to space constraints, we relegate the pseudo-code of the algorithm to Appendix~\ref{app:dapper}.

We analyze \textsc{Dapper} when the loss function is strongly convex in the hypothesis parameters $h$ and show that the model
minimizes the intended loss to the desired
accuracy. To the best of our knowledge, this is the first fine-tuning algorithm with provable guarantees.

To prove convergence guarantees, we need to ask what the desired
convergence guarantee is. Usually, models are required to converge to
the generalization guarantee and we use the same criterion. To this
end, we first state a known generalization theorem.  Let
$h_{\h\lambda} = \argmin \lambda \sL_{\h{\sD}_k}(h) + (1-\lambda)
\sL_{\h{\sC}}(h)$ and $h_{\lambda} = \argmin \lambda
\sL_{{\sD_k}}(h) + (1-\lambda) \sL_{{\sC}}(h)$.
\begin{lemma}[~\citep{blitzer2008learning}]
If the pseudo-dimension of the $\sH$ is $d$ \footnote{\cite{blitzer2008learning} states the result for $0-1$ loss, but it can extended to other loss functions.}, then with
probability at least $1- \delta$,
\begin{align*}
\conf{&} \lambda \sL_{{\sD_k}}(h_{\h\lambda}) + (1-\lambda)
\sL_{{\sC}}(h_{\h\lambda}) - \lambda \sL_{{\sD_k}}(h_\lambda) +
(1-\lambda) \sL_{{\sC}}(h_\lambda) \conf{\\&} = \cO \left(
\sqrt{\frac{\lambda^2}{m_k} + \frac{(1-\lambda)^2}{m_C}} \cdot
\sqrt{d\log \frac{1}{\delta}} \right).
\end{align*}
\end{lemma}
Since the generalization bound scales as $ \sqrt{\frac{\lambda^2}{m_k}
  + \frac{(1-\lambda)^2}{m_C}}$, the same accuracy in convergence is
desired.  Let
$
\epsilon_\lambda = \sqrt{\frac{\lambda^2}{m_k} + \frac{(1-\lambda)^2}{m_C}},
$
denote the desired convergence guarantee.
For strongly convex functions, we show that one can achieve this
desired accuracy using \textsc{Dapper}, furthermore the amount of
additional data is a constant multiple of $m_k$, independent of
$\lambda$ and $m$.
\begin{theorem}[Appendix~\ref{app:data}]
\label{thm:data}
Assume that the loss function is $\mu$-strongly convex and assume that
the gradients are $G$-smooth. Let $\sH$ admit diameter at most $R$. Let
$
r \geq G^2 \left(\frac{4G}{\mu} + 2 R \right)^2,
$
a constant independent of $\lambda$.  Let the learning rate $\eta =
\frac{1}{G\sqrt{rm_k}} \min \left(\frac{2G
  \lambda}{\mu(1-\lambda)}, R\right)$. Then after $r \cdot m_k$ steps of SGD,
the output $h_A$ satisfies,
\begin{align*}
\E[ \lambda \sL_{\h{\sD}_k}(h_A) + (1-\lambda) \sL_{\h{\sC}}(h_A)]
\conf{\\} \leq \E[\lambda \sL_{\h{\sD}_k}(h_{\h{\lambda}}) +
  (1-\lambda) \sL_{\h{\sC}}(h_{\h{\lambda}})] + \epsilon_\lambda.
\end{align*}
\end{theorem}

The above bound shows the convergence result for a given $\lambda$. One can find the best $\lambda$ either by cross validation or by minimizing the overall generalization bound of \cite{blitzer2008learning}.  While the above algorithm  reduces the amount 
of data transfer and is computationally efficient,
it may be vulnerable to privacy issues in applications such as FL. We propose several alternatives to overcome these privacy issues in Appendix~\ref{sec:practical}.

\section{Model interpolation}

The above approaches assume that the final inference model belongs to class $\cH$. In
practice, this may not be the case. One can learn a central model $h_c$ from a class $\cH_c$, and learn a local model $h_l$ from $\cH_l$, and use their interpolated model
\[
\lambda \cdot h_{l} + (1-\lambda) \cdot h_{c}.
\] 
More formally, let $h_c$ be the central or cluster model and let $\bar{h}_l
= (h_{l,1}, h_{l,2},\ldots, h_{l,p})$, where $h_{l,k}$ is the local model for client $k$.
Let $\lambda_k$ be the interpolated weight for client $k$ and let $\bar{\lambda}
= \lambda_1, \lambda_2,\ldots, \lambda_p$. 
If one has access to the true distributions,
then learning the best interpolated models can be formulated as the following optimization, 
\[
\min_{h_c, \bar{h}_l, \bar{\lambda}} \sum^p_{k=1} \frac{m_k}{m} \sL_{{\sD_k}}( (1-\lambda_k)h_c + \lambda_k h_{l,k}).
\]
Since, the learner does not have access to the true distributions, we replace $\sL_{{\sD_k}}( (1-\lambda_k)h_c + \lambda_k h_{l,k})$ with $\sL_{\h \sD_k}( (1-\lambda_k)h_c + \lambda_k h_{l,k})$ in the above optimization.
We now show a generalization bound for the above optimization.
\begin{theorem}[Appendix~\ref{app:model}]
\label{thm:model}
Let the loss $\ell$ is $L$ Lipschitz, $\sH_c$ be the
hypotheses class for the central model, and $\sH_\ell$ be the
hypotheses class for the local models. Let $h^*_c, \bar{\lambda}^*,
\bar{h}^{*}_{l}$ be the optimal values and $\h{h}^*_c,
\h{\lambda}^*_k, \h{h}^*_l$ be the optimal values for the empirical
estimates. Then, with probability at least $1- \delta$,
\begin{align}
& \sum^p_{k=1} \frac{m_k}{m} \sL_{{\sD_k}}(
  (1-\h{\lambda}^*_k)\h{h}^*_c + \h{\lambda}^*_k \h{h}^*_{l,k})
  \conf{\\&} - \sum^p_{k=1} \frac{m_k}{m} \sL_{{\sD_k}}(
  (1-{\lambda^*_k}){h^*_c} + \lambda^*_k {h^{*}_{l,k}}) \label{eq:mi} \\ & \leq 2L
  \left( \R_{\sU, m}(\sH_c) + \sum^p_{k=1} \frac{m_k}{m} \R_{\sD_k,
    m_k} (\sH_l) \right) + 2 \sqrt{\frac{\log \frac{1}{\delta}}{m}}. \nonumber
\end{align}
\end{theorem}
Standard bounds on Rademacher complexity by the pseudo-dimension yields the following corollary.
\begin{corollary}
Assume that $\ell$ is L Lipschitz. Let $h^*_c, \bar{\lambda}^*,
\bar{h}^{*}_{l}$ be the optimal values and $\h{h}^*_c,
\h{\lambda}^*_k, \h{h}^*_l$ be the optimal values for the empirical
estimates.  Then with probability at least $1- \delta$, the LHS of~\eqref{eq:mi} is bounded by
\begin{align*}
2L
  \left( \sqrt{\frac{d_c}{m}\log \frac{em}{d_c}} + \sqrt{\frac{d_lp}{m}\log
    \frac{em}{d_l}}\right) + 2\sqrt{\frac{\log \frac{1}{\delta}}{m}},
\end{align*}
where $d_c$ is the pseudo-dimension of $\sH_c$ and $d_l$ is the
pseudo-dimension of $\sH_l$.
\end{corollary}
Hence for models to generalize well, it is desirable to have $m \gg
d_c$ and the average number of samples to be much greater than $d_l$,
i.e., $m/p \gg d_l$. 
Similar to
Corollary~\ref{cor:cluster}, this bound only depends on the average
number of samples and not the minimum number of samples.

A common approach for model interpolation in practice is to first
train the central model $h_c$ and then train the local model $h_l$
separately and find the best interpolation coefficients, i.e.,
\[
\h{h}_c = \argmin_{h_c} \sum^p_{k=1} \frac{m_k}{m} \sL_{\h{\sD}_k}(h_c
)\conf{,}\arxiv{\text{\quad and \quad}} \h{h}_{l,k} = \argmin_{h_{l,k}} \sL_{\h{\sD}_k} (h_{l,k} ).
\]
We show that this approach might not be optimal in some instances and also propose a joint optimization for minimizing local and global models. Due to space constraints, we refer reader to Appendix~\ref{sec:mapper} for details. We refer to model interpolation algorithms by~\textsc{Mapper}.

\begin{table}[t]
  \centering
  \caption{Test loss of \textsc{HypCluster} as a function of number of clusters $q$ for the synthetic dataset.}
  \begin{tabular}{c c c c c c} 
    $q$ &  $1$ & $2$ & $3$ & $4$ & $5$  \\ \hline
    test loss &  $3.4$ & $3.1$ &  $2.9$ &  $2.7$  &  $2.7$  
  \end{tabular}
\label{tab:toy_cluster}
\end{table}
\begin{table}[H]
  \centering
  \caption{Test accuracy of seen clients for the EMNIST dataset.}
\label{tab:seen}
  \begin{tabular}{l c c c c} 
  	& initial model & +\textsc{Finetune} & +\textsc{Dapper} & +\textsc{Mapper}  \\ \hline
  	\textsc{FedAvg} & 84.3\% & 90.0\% & 90.1\% & 90.0\% \\
  		\textsc{Agnostic} & 84.6\% & 89.9\% & 90.0\% & 89.9\% \\
  	\textsc{HypCluster} $(q=2)$ & 89.1\% & 90.2\% & \textbf{90.3\%} & 90.1\% \\
  \end{tabular}
\end{table}
\begin{table}[H]
  \centering
  \caption{Test accuracy of unseen clients for the EMNIST dataset.}
\label{tab:unseen}
  \begin{tabular}{l c c c c} 
  	& initial model & +\textsc{Finetune} & +\textsc{Dapper} & +\textsc{Mapper}  \\ \hline
  	\textsc{FedAvg} & 84.1\% & \textbf{90.3\%} & \textbf{90.3\%} & 90.2\% \\
  		\textsc{Agnostic} & 84.5\% & 90.1\% & 90.2\% & 90.1\% \\
  	\textsc{HypCluster}  $(q=2)$ & 88.8\% & 90.1\% & 90.1\% & 89.9\% \\
  \end{tabular}
\end{table}

\section{Experiments}

\subsection{Synthetic dataset}

We first demonstrate the proposed algorithms on a synthetic dataset
for density estimation.
    Let $\sX = \emptyset$, $\sY = [d]$, and
$d=50$. Let $\ell$ be cross entropy loss and the number of users
$p=100$.  We create client distributions as a mixture of a uniform component, a
cluster component, and an individual component. 
 \begin{wrapfigure}{r}{0.34\textwidth}
\centering
   \caption{Test loss of  algorithms as a function of number of
    samples per user for the synthetic dataset.}
  \label{fig:toy}
 \includegraphics[width=0.34\columnwidth]{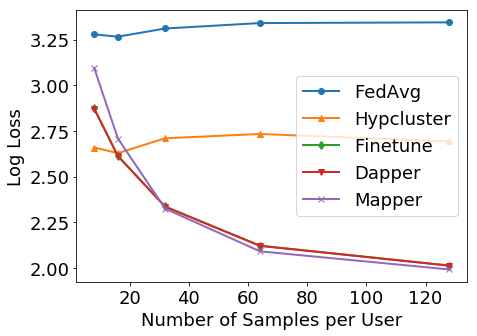}
    \end{wrapfigure}
    The details of the distributions
are in Appendix~\ref{app:toy}. We evaluate the algorithms as we vary the number of
samples per user. The results are in Figure~\ref{fig:toy}. \textsc{HypCluster} performs the best
when the number of samples per user $m_k$ is very small. If $m_k$ is large, \textsc{Mapper} performs the best followed closely by \textsc{Finetune} and \textsc{Dapper}. However, the difference between \textsc{Finetune} and \textsc{Dapper} is statistically insignificant.
In order to understand the effect of clustering, we evaluate various
clustering algorithms as a function of $q$ when $m_k
=100$, and the results are in Table~\ref{tab:toy_cluster}. Since the
clients are naturally divided into four clusters, as we increase $q$,
the test loss steadily decreases till the number of clusters reaches $4$
and then remains constant.

\subsection{EMNIST dataset}
 We  evaluate the proposed algorithms on the federated
\textsc{EMNIST-62} dataset~\citep{caldas2018leaf} provided by TensorFlow Federated
(TFF). The dataset consists of 3400 users' examples that are each one of 62 classes. We select 2500 users to train the global models (referred to as seen) and leave the remaining 900 as unseen clients reserved for evaluation only. We shuffle the clients first before splitting as the original client ordering results in disjoint model performance. 
The reported metrics are uniformly
averaged across clients similar to previous
works~\citep{jiang2019improving}.  For model architecture, we use a two-layer convolutional neural net. We refer to
Appendix~\ref{app:emnist} for more details on the architecture and training procedure.

The test results for seen and unseen clients are in Table~\ref{tab:seen} and Table~\ref{tab:unseen}, respectively. We trained models with \textsc{FedAvg}, \textsc{Agnostic}~\citep{mohri2019agnostic}, and \textsc{HypCluster} and combined them with \textsc{Finetune}, \textsc{Dapper}, and \textsc{Mapper}.
We observe that
\textsc{HypCluster} with two clusters performs significantly better compared to \textsc{FedAvg} and \textsc{Agnostic} models and improves accuracy by at least $4.3\%$.  Thus clustering is significantly better than training a single global model. 

The remaining algorithms \textsc{Dapper} and \textsc{Mapper} improve the accuracy by another $~1\%$ compared to \textsc{HypCluster}, but the EMNIST dataset is small and standard deviation in our experiments was in the order of $0.1\%$ and hence their improvement over \textsc{Finetune} is not statistically significant. However, these algorithms have provable generalization guarantees and thus would be more risk averse.

\section{Conclusion}

We presented a systematic learning-theoretic study of personalization
in learning and proposed and analyzed three algorithms: user
clustering, data interpolation, and model interpolation. For all three
approaches, we provided learning theoretic guarantees and efficient
algorithms. Finally, we empirically demonstrated the usefulness of the
proposed approaches on synthetic and EMNIST datasets.

\section{Acknowledgements}

Authors thank Rajiv Mathews, Brendan Mcmahan, Ke Wu, and Shanshan Wu for helpful comments and discussions.

\conf{
\section{Broader impact}
In this work, we provide algorithms and theoretical guarantees for personalization in machine learning and federated learning. While our work is theoretical in nature and does not present any immediate foreseeable societal consequences, we hope that having better personalization algorithms encourage a broader adaptation of personalization and federated learning. 
}

\bibliography{personal}
\conf{
\bibliographystyle{abbrvnat}
}
\newpage
\appendix

\section{Related works}

\subsection{Federated learning}
\label{app:federated}
FL was introduced by~\citet{mcmahan2017communication} as an efficient
method for training models in a distributed way. They proposed a new communication-efficient
optimization algorithm called \emph{FedAvg}.  They also showed that the training procedure
provides additional privacy benefits. The introduction of FL has given
rise to several interesting research problems, including the design of
more efficient communication strategies
\citep{konevcny2016federated,konecny2016federated2,
  suresh2017distributed, stich2018local, karimireddy2019scaffold}, the
study of lower bounds for parallel stochastic optimization with a
dependency graph \citep{woodworth2018graph}, devising efficient
distributed optimization methods benefiting from differential privacy
guarantees \citep{agarwal2018cpsgd}, stochastic optimization solutions
for the agnostic formulation \citep{mohri2019agnostic}, and
incorporating cryptographic techniques \citep{bonawitz2017practical}, meta-learning~\citep{chen2018federated},
see \citep{li2019federatedsurvey, kairouz2019advances} for an in-depth
survey of recent work in FL.

Federated learning often results in improved performance, as reported in several learning problems, including next word prediction \citep{hard2018federated,  yang2018applied}, vocabulary estimation
\citep{chen2019federated}, emoji prediction \citep{ramaswamy2019federated}, decoder models \citep{chen2019federatedb}, low latency vehicle-to-vehicle
communication \citep{samarakoon2018federated}, and predictive models
in health \citep{brisimi2018federated}.

\subsection{Personalization in federated learning}
\label{sec:related}

There are several recent works that focus on multi-task and meta-learning in the context of federated learning. \citet{smith2017federated} studied the problem of federated multi-task learning and proposed MOCHA, an algorithm that jointly learns parameters and a similarity matrix between user tasks. MOCHA tackles various aspects of distributed multitask learning including communication constraints, stragglers, and fault tolerance.  They focus on the convex setting and 
their application to non-convex deep learning models where strong duality is no longer guaranteed is unclear.

\citet{jiang2019improving} drew interesting connections between FedAvg and first-order model agnostic meta-learning (MAML) \citep{finn2017model}
and showed that FedAvg is in fact already a
meta-learning algorithm. \citet{fallah2020personalized} also proposed to use the MAML objective in a global model training to obtain a better personalizable global model. \citet{khodak2019adaptive} proposed ARUBA that improves upon gradient-based meta-learning approaches.
A variational approach for multi-task learning was proposed by \cite{corinzia2019variational}. 
Recently, \cite{hanzely2020federated} proposed to learn a model per user by adding an $\ell_2$ penalty on model parameters to ensure they are similar. 

Another line of work uses a set of local parameters, which are trained per-client,
and a set of global parameters, which are trained using FL. For example, 
\citet{bui2019federated} proposed to use user-representations by having a set of client-specific local parameters, which are trained per-client and a set of global parameters, which are trained using FL. \cite{arivazhagan2019federated, liang2020think} proposed to store some layers of the model locally, while training the rest of the model with federated learning.

\cite{peterson2019private} proposed to use techniques from a mixture of experts literature and their approach is similar to our approach of model interpolation, but they learn an interpolation weight based on features. 
Furthermore, there are no theoretical guarantees for their approach. A theoretical analysis of interpolation models without variable mixing weights was recently presented in \cite{agarwal2020federated}. Concurrent to this work, \cite{deng2020adaptive} proposed to use an interpolation of a local and global model. Their approach is similar to model interpolation in our paper.

\cite{wang2019federated} showed that federated models can be
fine-tuned based on local data. They proposed methods to find the best
hyper-parameters for fine-tuning and showed that it improves the next
word prediction of language models in virtual keyboard applications. \citet{yu2020salvaging} proposed several variants of the fine-tuning approach, including training only 
a few layers of the networks, adding a local penalty term in the form of model distillation,
or elastic weight averaging to the fine-tuning objective to improve local adaptation.

\cite{zhao2018federated} showed that one can improve the accuracy of FedAvg, by sharing a small amount of  public data to reduce the non-i.i.d.\ nature of the client data. \cite{sattler2019clustered} proposed to use cosine-similarity between gradient updates for clustering in federated learning. 
However, their approach requires all clients to participate in each round and 
hence is computationally infeasible. Personalization in other settings such as 
peer-to-peer networks have been studied by \cite{zantedeschi2019fully}. We refer the readers to \citep{kulkarni2020survey} for a survey of algorithms for personalization in FL.

\section{Global models}

\subsection{Example for the suboptimality of global models}
\label{app:example}

We
provide the following simple example, which shows that global models
can be a constant worse compared to the local model.

\begin{example}
Let $\cX = \Rset$ and $\cY = \{0, 1\}$. Suppose there are two clients
with distributions $\sD_1$ and $\sD_2$ defined as follows. $ \forall
x, \sD_1(x) = \sD_2(x)$ and $\sD_1(1 | x) = 1 $ if $ x> 0$ and zero
otherwise. Similarly, $\sD_2(1 | x) = 1 $ only if $ x < 0$ and zero
otherwise. Let $\cH$ be the class of threshold classifiers indexed by
a threshold $t \in \Rset$ and sign $s \in \{-1, 1\}$ such that
$h_{t,s} \in \cH$ is given by $h_{t,s}(x) = \indic_{(x - t)s > 0}$.
Further, suppose we are interested in zero-one loss and the number of
samples from both  domains is very large and equal.

The optimal classifier for $\sD_1$ is $h_{0,1}$ and the optimal
classifier for $\sD_2$ is $h_{0,-1}$, and they achieve zero error in
their respective clients. Since the number of samples is the same from both clients, $\sU$ is
the uniform mixture of the two domains, $\sU = 0.5 \sD_1 + 0.5 \sD_2$
. Note that for all $h \in \sH$, $\sL_{\sU}(h) = 0.5$ and hence the
global objective cannot differentiate between any of the hypotheses in
$\cH$. Thus, with high probability, any globally trained model incurs a
constant loss on both clients.
\end{example}

\subsection{Agnostic global model}
\label{app:agnostic}

Instead of assigning weights proportional to the number of
samples as in the uniform global model, we can weight them according to any $\lambda \in
\Delta^p$. For example, instead of uniform sample weights, we
can weight clients uniformly corresponding to $\lambda_k =
\frac{1}{p}$, for all $k$. Let $\bar \sD_\lambda$ denote the $\lambda$-weighted
empirical distribution and let $h_{\bar \sD_\lambda}$ be the minimizer
of loss over $\bar \sD_\lambda$.
Instead of the uniform global model described in the previous section, 
we can use the agnostic
loss, where we minimize the
maximum loss over a set of distributions. Let $\Lambda \subseteq
\Delta^p$. Agnostic loss is given by
\[
 \max_{\lambda \in \Lambda } \sL_{\bar{ \sD}_\lambda}(h).
\]
Let $h_{\bar{\sD}_{\Lambda}}$ be the minimizer.  Let $\s (\Lambda,
\bm) = \max_{\lambda\in \Lambda} \s (\lambda, \bm) $.  Let
$\Lambda_\epsilon$ be an $\epsilon$-cover of $\Delta^p$. Let $\bm$ denote the empirical distribution of samples $(m_1/m,
m_2/m,\ldots, m_p/m)$. The skewness
between the distributions $\lambda$ and $\bm$ is defined as $
\s (\lambda, \bm) =  \sum^p_{k=1} \frac{\lambda^2_k}{\bm_k},
$
where $\bm_k = m_k / m$.
With these
definitions, the generalization guarantee of \citep[Theorem 2]{mohri2019agnostic} 
for client one can be expressed as follows:
\begin{align*}
\conf{&} \sL_{\sD_1}(h_{\h{\sD}_{\Lambda_\epsilon}}) \leq
\sL_{\sD_1}(h_{\sD_1}) + \tcO \left( \sqrt{\s(\Lambda_\epsilon || \bm)} \cdot
\frac{\sqrt{d + \log
    \frac{|\Lambda_\epsilon|}{\delta}}}{\sqrt{m}} + \epsilon \right)
+ \disc_{\cH}(\sD_1, \sD_{\lambda*}) ,
\end{align*}
where $\lambda^* = \argmax_{\lambda} \sL_{\bar{ \sD}_\lambda}(h_{\h{\sD}_{\Lambda_\epsilon}})$ is the mixture weight
where the trained model $h_{\h{\sD}_{\Lambda_\epsilon}}$ has the highest loss.
Hence, this approach would personalize well for hard
distributions and can be considered as a step towards ensuring that models
work for all distributions. In this work, we show that training a
different model for each client would significantly improve the model
performance.

\section{Supplementary material for clustering}

\subsection{Baselines}
\label{app:cluster_baseline}

 If we have meta-features about the
data samples and clients, such as location or type of
device, we can use them to find clusters. This can be achieved
by algorithms such as $k$-means or variants.
This approach depends on the knowledge of the meta-features and their
relationship to the set of hypotheses under consideration. While it
may be reasonable in many circumstances, it may not be always
feasible.  If there are no meta-features, a natural approach is to cluster using
a Bregman divergence defined over the distributions
$\sD_k$~\citep{banerjee2005clustering}. 
However, it is likely that we would overfit as the generalization
of the density estimation depends on the covering number of the class of distributions
$\sD_1, \sD_2,\ldots, \sD_p$, which in general can be much larger than
that of the class of hypotheses $\cH$. To overcome this, we propose an
approach based on hypotheses under consideration which we discuss
next.

\subsection{Generalization of clustering algorithms}
\label{app:lem_cluster}
Recall that we solve for 
\begin{equation}
    \label{eq:emp_cluster}
    \min_{h_1, \ldots, h_q} \sum^p_{k=1} \frac{m_k}{m} \cdot
    \min_{i \in [q]} \sL_{\h{\sD}_k}(h_i).
\end{equation}
\begin{lemma}
\label{lem:cluster}
Let $h^*_1, h^*_2, \ldots, h^*_q$ be the $q$ models obtained by
solving~\eqref{eq:cluster} and $\h{h}^*_1, \h{h}^*_2, \ldots,
\h{h}^*_q$ be the $q$ models obtained by
solving~\eqref{eq:emp_cluster}. Then,
\begin{align}
  &  \sum^p_{k=1} \frac{m_k}{m} \cdot \left( \min_{i \in [q]} \sL_{\sD_k}(\h{h}^*_i)  
  -\min_{i \in [q]}
 \sL_{\sD_k}(h^*_i) \right) 
   \leq 2 \max_{h_1, \ldots, h_q} \left \lvert  \sum^p_{k=1} \frac{m_k}{m} \cdot \left( \min_{i \in [q]} \sL_{\sD_k}(h_i)  -  
  \min_{i \in [q]} \sL_{\h{\sD}_k}(h_i)  \right) \right \rvert. \label{eq:2}
\end{align}
\end{lemma}
\begin{proof}
\begin{align*}
&  \sum^p_{k=1} \frac{m_k}{m} \cdot \min_{i \in [q]} \sL_{\sD_k}(\h{h}^*_i)  
 - \sum^p_{k=1} \frac{m_k}{m} \cdot \min_{i \in [q]} \sL_{\sD_k}(h^*_i) \nonumber \\
& = \sum^p_{k=1} \frac{m_k}{m} \cdot \min_{i \in [q]} \sL_{\sD_k}(\h{h}^*_i)  
 +  \sum^p_{k=1} \frac{m_k}{m} \cdot \min_{i \in [q]} \sL_{\h{\sD}_k}(\h{h}^*_i) 
 -  \sum^p_{k=1} \frac{m_k}{m} \cdot \min_{i \in [q]} \sL_{\h{\sD}_k}(\h{h}^*_i) 
   \nonumber \\
&  + \sum^p_{k=1} \frac{m_k}{m} \cdot \min_{i \in [q]} \sL_{\h{\sD}_k}(h^*_i) 
- \sum^p_{k=1} \frac{m_k}{m} \cdot \min_{i \in [q]} \sL_{\sD_k}(h^*_i)
 - \sum^p_{k=1} \frac{m_k}{m} \cdot \min_{i \in [q]} \sL_{\h{\sD}_k}(h^*_i) \nonumber \\
&  \leq 2 \max_{h_1, \ldots, h_q} \left \lvert  \sum^p_{k=1} \frac{m_k}{m} \cdot \min_{i \in [q]} \sL_{\sD_k}(h_i)  -  \sum^p_{k=1} \frac{m_k}{m} \cdot \min_{i \in [q]} \sL_{\h{\sD}_k}(h_i)  \right \rvert, 
 \end{align*}
 where the inequality follows by observing that $\sum^p_{k=1} \frac{m_k}{m} \cdot \min_{i \in [q]} \sL_{\h{\sD}_k}(\h{h}^*_i) \leq  \sum^p_{k=1} \frac{m_k}{m} \cdot \min_{i \in [q]} \sL_{\h{\sD}_k}(h^*_i)$, by the definition of $\h{h}^*_i$.
 \end{proof}
\subsection{Proof of Theorem~\ref{thm:cluster}}
\label{app:cluster}

For any set of real numbers $a_1, a_2, \ldots, a_q$ and $b_1, b_2, \ldots, b_q$, observe that
\begin{align*}
\min_i a_i -   \min_i b_i 
& = \min_i b_i + (a_i-b_i)  -  \min_i b_i  \leq  \min_i b_i + \max_i (a_i-b_i)  -  \min b_i  = \max_i (a_i-b_i).
\end{align*}
We first prove the theorem for one side. 
Let $f \colon[p] \to [q]$ be a  mapping from clients to clusters. Applying the above 
result yields,
\begin{align*}
& \max_{h_1, \ldots, h_q} \left( \sum^p_{k=1} m_k \cdot \min_{i \in [q]} \sL_{\sD_k}(h_k)  -  \sum^p_{k=1} m_k \cdot \min_{i \in [q]} \sL_{\h{\sD}_k}(h_k) \right) \\
& \leq   \max_{h_1, \ldots, h_q} \left(  \sum^p_{k=1} m_k \cdot \max_{k \in [p]} (\sL_{\sD_k}(h_k) - \sL_{\h{\sD}_k}(h_k)) \right) \\
 & =   \max_{h_1, \ldots, h_q} \left(  \sum^p_{k=1} m_k \cdot \max_{f(k)} (\sL_{\sD_k}(h_{f(k)}) - \sL_{\h{\sD}_k}(h_{f(k)})) \right) \\
 & =  \max_{h_1, \ldots, h_q} \max_{f} \left(  \sum^p_{k=1} m_k \cdot  (\sL_{\sD_k}(h_{f(k)}) - \sL_{\h{\sD}_k}(h_{f(k)})) \right) \\
 & = \max_{f}  \max_{h_1, \ldots, h_q}  \left(  \sum^p_{k=1} m_k \cdot  (\sL_{\sD_k}(h_{f(k)}) - \sL_{\h{\sD}_k}(h_{f(k)})) \right).
  \end{align*}
Since changing one sample changes the above function by at most $1$, for a given $f$, by the McDiarmid's inequality, with probability at least  $1-\delta$, the following holds:
\begin{align*}
 &  \max_{h_1, \ldots, h_q}  \left(  \sum^p_{k=1} m_k \cdot  (\sL_{\sD_k}(h_{f(k)}) - \sL_{\h{\sD}_k}(h_{f(k)})) \right)  \\
  & 
  \leq \E \left[   \max_{h_1, \ldots, h_q}  \left(  \sum^p_{k=1} m_k \cdot  (\sL_{\sD_k}(h_{f(k)}) - \sL_{\h{\sD}_k}(h_{f(k)})) \right) \right] + 2\sqrt{m \log \frac{1}{\delta}}.
\end{align*}
The number of  possible functions $f$ is $q^p$. Hence, by the union bound, for all $f$, with probability at least $1-\delta$, the following holds:
\begin{align*}
&  \max_{h_1, \ldots, h_q}  \left(  \sum^p_{k=1} m_k \cdot  (\sL_{\sD_k}(h_{f(k)}) - \sL_{\h{\sD}_k}(h_{f(k)})) \right)  \\
&  \leq \E \left[   \max_{h_1, \ldots, h_q}  \left(  \sum^p_{k=1} m_k \cdot  (\sL_{\sD_k}(h_{f(k)}) - \sL_{\h{\sD}_k}(h_{f(k)})) \right) \right] + 2 \sqrt{m p\log \frac{q}{\delta}}.
\end{align*}
For a given clustering $f$, by the sub-additivity of $\max$, 
\begin{align*}
   &    \max_{h_1, \ldots, h_q}  \left(  \sum^p_{k=1} m_k \cdot  (\sL_{\sD_k}(h_{f(k)}) - \sL_{\h{\sD}_k}(h_{f(k)})) \right) \\
      & =  \max_{h_1, \ldots, h_q}  \left(\sum^p_{k=1}  \sum_{k: f(k) = i} m_k \cdot  (\sL_{\sD_k}(h_{f(k)}) - \sL_{\h{\sD}_k}(h_{f(k)})) \right) \\
      & \leq \sum^{q}_{i=1}  \max_{h_1, \ldots, h_q}  \left( \sum_{k: f(k) = i} m_k \cdot  (\sL_{\sD_k}(h_{f(k)}) - \sL_{\h{\sD}_k}(h_{f(k)})) \right) \\
          & = \sum^{q}_{i=1}  \max_{h_1, \ldots, h_q}  \left( \sum_{k: f(k) = i} m_k \cdot  (\sL_{\sD_k}(h_i) - \sL_{\h{\sD}_k}(h_i)) \right) \\
                   & = \sum^{q}_{i=1}  \max_{h_i}  \left( \sum_{k: f(k) = i} m_k \cdot  (\sL_{\sD_k}(h_{i}) - \sL_{\h{\sD}_k}(h_{i})) \right) \\
                   & = \sum^{q}_{i=1}  \max_{h_i}  \left( m_{C_i} \cdot  (\sL_{\sC_i}(h_{i}) - \sL_{\h{\sC}_i}(h_{i})) \right),
\end{align*}
where $C_i$ is the cluster of clients such that $f(k) = i$ and $m_{C_i}$ is the
number of samples in that cluster, and $\sC_i$ is its distribution.
Thus,
\begin{align*}
 \E \left[  \max_{h_1, \ldots, h_q}  \left(  \sum^{p}_{k=1} m_k \cdot  (\sL_{\sD_k}(h_{f(k)}) - \sL_{\h{\sD}_k}(h_{f(k)})) \right)\right] 
& \leq \E \left[ \sum^{q}_{i=1}  \max_{h_i}  \left( m_{C_i} \cdot  (\sL_{\sC_i}(h_i) - \sL_{\h{\sC}_i}(h_i)) \right) \right] \\
& \leq \sum^{q}_{i=1}  \R_{\sC_i, m_{C_i}}(\sH) m_{C_i},
\end{align*}
where the last inequality follows from standard learning-theoretic guarantees
and the definition of Rademacher complexity~\citep{mohri2018foundations}. The proof 
follows by combining the above equations, normalizing by $m$, and the union bound.

\subsection{Proof of Corollary~\ref{cor:cluster}}
\label{app:cor_cluster}
We show that for any clustering,
\[
 \sum^q_{i=1} \frac{m_{C_i}}{m} \R_{\sC_i, m_{C_i}}(\sH)
 \leq  \sqrt{\frac{dp}{m}\log \frac{em}{d}}.
\]
The proof then follows from Theorem~\ref{thm:cluster}.
To prove the above observation, observe that
\begin{align*}
 \sum^q_{i=1} \frac{m_{C_i}}{m} \R_{\sC_i, m_{C_i}}(\sH) 
 &\leq \sum^q_{i=1} \frac{m_{C_i}}{m} \sqrt{\frac{dp}{m_{C_i} }\log \frac{em_{C_i}}{d}} 
  \leq \sum^q_{i=1} \frac{m_{C_i}}{m} \sqrt{\frac{dp}{m_{C_i}}\log \frac{em}{d}} \\
&   \leq \sum^q_{i=1} \frac{1}{m} \sum^q_{i=1}  \sqrt{{dp}{m_{C_i} }\log \frac{em}{d}}
  \leq \sqrt{\frac{dq}{m}\log \frac{em}{d}},
\end{align*}
where the last inequality follows from Jensen's inequality.

\section{Supplementary material for data interpolation}

\subsection{Proof of Theorem~\ref{thm:data}}
\label{app:data}

Let  $g(h) = \lambda \sL_{\h{\sD}_k}(h) + (1-\lambda) 
    \sL_{\h{\sC}}(h)$. Suppose we are interested in running $T$ steps of SGD on $g$, where at each step we independently
    sample $\h{\sD}_k$ with probability $\lambda$
    and $\h{\sC}$ with probability $1-\lambda$ and choose a random sample from the selected empirical distribution to compute the gradient.
    This can be simulated by first sampling $T$ elements from $\h{\sC}$, denoted by $\h{\sC}'$ and using $\h{\sC}'$ instead of $\h{\sC}$ during optimization. Hence to prove the theorem, suffices to show that $T=rm_k$ steps of SGD on $g$ using the above mentioned sampling procedure yields the desired bound.

We now ask how  large should $T$ be to obtain
    error of $\epsilon_\lambda$. By
    standard stochastic gradient descent guarantees, the output $h_A$
    satisfies
    \[
  \E[g(h_A)] \leq \E[g(h_{\h{\lambda}})]
  + \frac{\| h_{c} - h_{\h{\lambda}}) \|^2}{2 \eta} + \frac{\eta G^2 T}{2}.
    \]
Since the loss is strongly convex and $h_c$ is optimal for $\sL_{\h{\sC}}(h_{c})$,
\[
\sL_{\h{\sC}}( h_{\h{\lambda}}) - \sL_{\h{\sC}}(h_{c})  \geq \frac{\mu}{2} \| h_{c} - h_{\h{\lambda}}) \|^2.
\]
Furthermore, since $ h_{\h{\lambda}}$ is optimal for a $\lambda$-mixture,
\begin{align*}
 \lambda   \sL_{\h{\sD}_k}( h_{\h{\lambda}})
    + (1-\lambda) \sL_{\h{\sC}}( h_{\h{\lambda}})
    \leq \lambda   \sL_{\h{\sD}_k}( h_{c})
    + (1-\lambda) \sL_{\h{\sC}}( h_{c}).
\end{align*}
Hence,
\begin{align*}
\frac{\mu}{2} \| h_{c} - h_{\h{\lambda}}) \|^2
&\leq 
  \sL_{\h{\sC}}( h_{\h{\lambda}})
 -  \sL_{\h{\sC}}( h_{c}) \\
 & \leq \frac{\lambda}{1-\lambda} 
 (\sL_{\h{\sD}_k}( h_{c}) -  \sL_{\h{\sD}_k}( h_{\h{\lambda}})) \\
 & \leq  \frac{G \lambda}{1-\lambda} \| h_{c} - h_{\h{\lambda}}) \|.
\end{align*}
Therefore,
\begin{align*}
  \| h_{c} - h_{\h{\lambda}}) \| \leq \min \left(\frac{2G \lambda}{\mu(1-\lambda)}, R\right)
  \end{align*}
Combining the above equations, we get
\[
 \E[g(h_A)] \leq \E[g(h_{\h{\lambda}})]
  + \frac{1}{2\eta} \min \left(\frac{2G \lambda}{\mu(1-\lambda)}, R\right)^2
 + \frac{\eta G^2 T}{2}.
\]
Substituting the learning rate and setting $T = rm_k$ yields
\[
 \E[g(h_A)] \leq \E[g(h_{\h{\lambda}})] + 
 \frac{G}{\sqrt{rm_k}} \min \left(\frac{2G \lambda}{\mu(1-\lambda)}, R\right).
\]
Hence if $r \geq G^2 \max_{\lambda} \min \left(\frac{2G
}{\mu(1-\lambda)}, \frac{R}{\lambda}\right)^2$, the above
bound is at most $\sqrt{\frac{\lambda^2}{m_k}} \leq \epsilon_\lambda$.  Note
that for any $\lambda$,
\[
\min \left(\frac{2G }{\mu(1-\lambda)}, \frac{R}{\lambda}\right)
\leq \frac{2G}{\mu(1-\lambda)} 1_{\lambda < 1/2} + \frac{R}{\lambda} 1_{\lambda \geq 1/2} \leq \frac{4G}{\mu} + 2R,
\]
hence the theorem.

\subsection{Dapper pseudo-code}
\label{app:dapper}
We provide pseudo-code for the \textsc{Dapper} algorithm in Figure~\ref{fig:dapper}.
\begin{figure}[t]
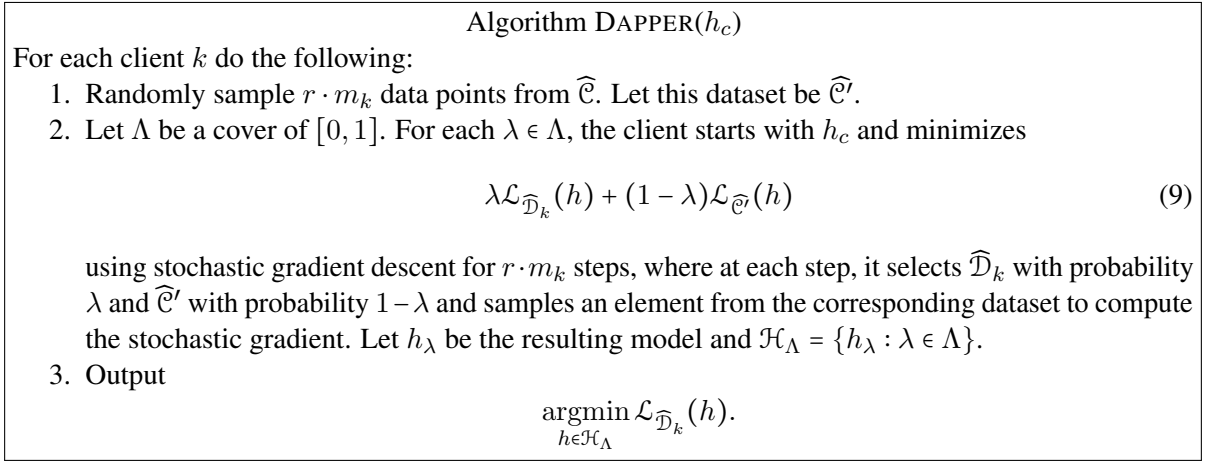

\begin{center}
\fbox{\begin{minipage}{\conf{0.45}\arxiv{0.95}\textwidth}
\begin{center}
Algorithm \textsc{Dapper}($h_c$)
\end{center}
 For each client $k$ do the following:
\begin{enumerate}
    \item Randomly sample $r \cdot m_k$ data points from
      $\h{\sC}$. Let this dataset be $\h{\sC'}$.
    \item Let $\Lambda$ be a cover of $[0,1]$. For each $\lambda \in
      \Lambda$, the client starts with $h_c$ and minimizes
    \begin{equation}
\label{eq:dapper}
    \lambda \sL_{\h{\sD}_k}(h) + (1-\lambda) \sL_{\h{\sC'}}(h)
    \end{equation}
    using stochastic gradient descent for $r \cdot m_k$ steps, where at each step, it selects $\h{\sD}_k$ with probability $\lambda$ and $\h{\sC}'$ with probability $1-\lambda$ and  samples an element from the corresponding dataset to compute the stochastic gradient. Let $h_\lambda$
    be the resulting model and $\sH_\Lambda = \{h_\lambda : \lambda
    \in \Lambda\}$.
\item Output 
\[
\argmin_{h \in \sH_{\Lambda}} \sL_{\h{\sD}_k}(h).
\]
\end{enumerate}
\end{minipage}}
\end{center}
\caption{Pseudocode for the \textsc{Dapper} algorithm.}
\label{fig:dapper}
\end{figure}

\subsection{Practical considerations}
\label{sec:practical}
While the above algorithm  reduces the amount 
of data transfer and is computationally efficient,
it may be vulnerable to privacy issues in applications such as FL.
To  overcome that, we propose several alternatives:
\begin{enumerate}
    \item Sufficient statistics: in many scenarios, instead of the actual data, we only need some sufficient statistics. For example in regression with $\ell^2_2$ loss, we only need the covariance matrix of the dataset from $\h{\sC}$.
    \item Generative models: for problems such as density estimation and language modelling, we can use the centralized model to generate synthetic samples from $h_c$ and use that as an approximation to $\h{\sC'}$. For other
    applications, one can train a GAN and send the GAN to the clients and the clients
    can sample from the GAN to create the dataset $\h{\sC'}$~\citep{augenstein2019generative}.
\item Proxy public data: if it is not feasible to send the actual user data, one could send proxy public data instead. While this may not be theoretically optimal, it will still avoid overfitting to the local data.
\end{enumerate}

\section{Supplementary material for model interpolation}
\subsection{Proof of Theorem~\ref{thm:model}}
\label{app:model}

Observe that
\begin{align*}
& \sum^p_{k=1} \frac{m_k}{m} \sL_{{\sD_k}}( (1-\h{\lambda}^*_k)\h{h}^*_c + \lambda^*_k \h{h}^{*}_{l,k})
 - \sum^p_{k=1} \frac{m_k}{m} \sL_{{\sD_k}}( (1-{\lambda^*_k}){h^*_c} + \lambda^*_k {h^{*}_{l,k}}) \\
 & \leq 2 \max_{h_c, \bar{\lambda}, \bar{h}_l}  
\left(\sum^p_{k=1} \frac{m_k}{m} \left( \sL_{{\sD_k}}( (1-{\lambda_k}){h_c} + \lambda_k {h_{l,k}})
- \sL_{\h{\sD}_k}( (1-{\lambda_k}){h_c} + \lambda_k {h_{l,k}}) \right)
\right).
\end{align*}
Changing one sample changes the above function by at most $1/m$. Thus by McDiarmid's inequality, with probability at least $1-\delta$,
\begin{align*}
& \max_{h_c, \bar{\lambda}, \bar{h}_l}  
\left(\sum^p_{k=1} \frac{m_k}{m} \left(\sL_{{\sD_k}}( (1-{\lambda_k}){h_c} + \lambda_k {h_{l,k}})
- \sL_{\h{\sD}_k}( (1-{\lambda_k}){h_c} + \lambda_k {h_{l,k}}) \right) \right)\\
& \leq \E \left[\max_{h_c, \bar{\lambda}, \bar{h}_l}  
\left(\sum^p_{k=1} \frac{m_k}{m}\left( \sL_{{\sD_k}}( (1-{\lambda_k}){h_c} + \lambda_k {h_{l,k}})
- \sL_{\h{\sD}_k}( (1-{\lambda_k}){h_c} + \lambda_k {h_{l,k}}) \right)\right) \right] + 2\sqrt{\frac{\log \frac{1}{\delta}}{m}}.
\end{align*}
Let $\bar{\sH}_l$ be the Cartesian product of hypothesis classes where the $k^\text{th}$
hypothesis class is the hypothesis applied to $k^\text{th}$ client. 
Let $\sH = \bar{\lambda} \sH_c + (1-\bar{\lambda}) \bar{\sH}_l$. 
Hence, by Talagrand's construction lemma and the properties of Rademacher complexity,
\begin{align*}
&  \E\left[\max_{h_c, \bar{\lambda}, \bar{h}_l}  
\left(\sum^p_{k=1} \frac{m_k}{m}\left( \sL_{{\sD_k}}( (1-{\lambda_k}){h_c} + \lambda_k {h_{l,k}})
- \sL_{\h{\sD}_k}( (1-{\lambda_k}){h_c} + \lambda_k {h_{l,k}}) \right) \right)\right]  \\
& \leq \R_{\sU + \bar{\sD}, \bar{m}}(\ell(\sH)) \\
& \leq  L \R_{\sU + \bar{\sD}, \bar{m}}(\sH)  \\
& \leq  L \R_{\sU, m}(\sH_c) + L \Rset_{\sD_k, m_k}(\bar{\sH}_l)  \\
& \leq  L \left( \R_{\sU, m}(\sH_c) + \sum^p_{k=1} \frac{m_k}{m} \R_{\sD_k, m_k} (\sH_l) \right),
\end{align*}
where the last inequality follows from the sub-additivity of Rademacher complexity.

\subsection{\textsc{Mapper} algorithms}
\label{sec:mapper}

We first show that this method of independently finding the local
models is sub-optimal with an example.
\begin{example}
Consider the following discrete distribution estimation problem. Let $\cH_c$ be
the set of distributions over $d$ values and let $\cH_l$ be the set of
distributions with support size $1$. For even $k$, let $\sD_k(1) =
\sD_{\text{even}}(1) = 1.0$ and for odd $k$, let $\sD_k(y) =
\sD_{\text{odd}}(y) = 1/d$ for all $1\leq y \leq d$. Let the number of
clients $p$ be very large and the number of samples per client a
constant, say ten. Suppose we consider the log-loss. 

The intuition behind this example is that since we have only one
example per domain, we can only derive good estimates for the local
model for even $k$ and we need to estimate the global model jointly
from the odd clients.  With this approach, the optimal solution is as
follows. For even $k$, $h_{l,k} = \sD_k$ and $\lambda_k = 1.0$. For
odd $k$, $\lambda_k = 0.0$ and the optimal $h_c$ is given by, $h_c =
\sD_{\text{odd}}$. If we learn the models separately, observe that, for each client
$\h{h}_{l,k}$ be the empirical estimate and $\h{h}_{c}$ would be $0.5
\cdot \sD_{\text{even}} + 0.5 \cdot \sD_{\text{odd}}$. Thus, for any
$\bar{\lambda}$, the algorithm would incur at least a constant loss
more than optimal for any $\lambda_k$ for odd clients.
\end{example}

\begin{figure}
\begin{center}
\fbox{\begin{minipage}{\conf{0.47}\arxiv{0.95}\textwidth}
\begin{center}
Algorithm \textsc{Mapper}
\end{center}
Randomly initialize $h^0_c$  and for $t = 1$ to $T$, randomly select a client $k$ and do the following.
\begin{enumerate}
    \item Let $\Lambda$ be a cover of $[0,1]$. For each $\lambda \in \Lambda$, 
let
    \begin{equation}
        \label{eq:mopper1}
   h_{l,k}(\lambda) = \argmin_{h_{l,k}} \sL_{\h{\sD}_k} (\lambda h_{l,k} + (1-\lambda) h^{t-1}_c).
    \end{equation}
    \item Find the best local model:
\begin{equation}
\label{eq:mopper2}
      \lambda^* = \argmin_{\lambda \in \Lambda} \sL_{\h{\sD}_k}
      (\lambda h_{l,k}(\lambda) + (1-\lambda) h^{t-1}_c).
\end{equation}
    \item Minimize the global model.
    \begin{equation}
    \label{eq:mopper3}
           h^{t}_c = h^{t-1}_c - \eta \nabla \sL_{\h{\sD}_k} (\lambda^*
           h_{l,k}(\lambda^*) + (1-\lambda^*) h^{t-1}_c).
    \end{equation}
\end{enumerate}
\item Let $h^T_c$ be the final global model. For each client $k$ rerun $1(a)$ and $1(b)$ to obtain the local model $h_{l,k}$ and the interpolation weight $\lambda_k$. 
\end{minipage}}
\end{center}
\caption{Pseudocode for the \textsc{Mapper} algorithm.}
\end{figure}

Since training models independently is sub-optimal in certain cases,
we propose a joint-optimization algorithm.  First observe that the
optimization can be rewritten as
\[
\min_{h_c} \sum^p_{k=1} \frac{m_k}{m} \min_{\bar{h}_l \bar{\lambda}}
\sL_{\h{\sD}_k}( (1-\lambda_k)h_c + \lambda_k h_{l,k}).
\]
 Notice that for a fixed $\lambda$ the function is convex in both
 $h^\ell$ and $h_c$. But with the minimization over $\lambda$, the
 function is no longer convex.  We propose algorithm \textsc{Mapper}
 for minimizing the interpolation models. At each round, the algorithm
 randomly selects a client. It then finds the best local model and
 interpolation weight for that client using the current value of the
 global model. It then updates the global model using the local model
 and the interpolation weight found in the previous step.

\section{Supplementary material for experiments}
\subsection{Synthetic dataset}
\label{app:toy}

Let $U$ be a uniform
distribution over $\sY$.  Let $P_k$ be a point mass distribution with
$P_k(k) = 1.0$ and for all $y \neq k$, $P_k(y) = 0$. 
For a client $k$, let $\sD_k$ be a mixture given by
\[
\sD_k = 0.5 \cdot P_{k\%4} + 0.25 \cdot U + 0.25 \cdot P_{k \%(d-4)},
\]
where $a \%b $, is $a$ modulo $b$.  Roughly, $U$ is the uniform
component and is same for all the clients, $P_{k\%4}$ is the cluster
component and same for clients with same clusters, and $P_{k\%(d-4)}$
is the individual component per client.

\subsection{EMNIST dataset}
\label{app:emnist}
\begin{table}
  \centering
  \caption{EMNIST convolutional model.}
  \begin{tabular}{c c c c c} 
  	Layer & Output Shape & \# Parameters & Activation & Hyperparameters  \\ \hline
  	Conv2d & (26, 26, 32) & 320 & ReLU & out\_chan=32;filter\_shape=(3, 3) \\
  	Conv2d & (24, 24, 64) & 18496 & ReLU & out\_chan=64;filter\_shape=(3, 3) \\
  	MaxPool2d & (12, 12, 64) & 0 & & window\_shape=(2, 2);strides=(2, 2) \\
  	Dropout & (12, 12, 64) & 0 & & keep\_rate=0.75 \\
  	Flatten & 9216 & 0 & & \\
  	Dense & 128 & 1179776 & ReLU & \\
  	Dropout & 128 & 0 & & keep\_rate=0.5 \\
  	Dense & 62 & 7998 & LogSoftmax & \\
  \end{tabular}
\label{tab:emnist_layers}
\end{table}

For the EMNIST experiments, we follow previous work for model architecture~\citep{jiang2019improving}. The full model architecture layer by layer is provided in Table~\ref{tab:emnist_layers}.
We train the model for 1000 communication rounds with 20 clients per round and use server side momentum, though one can use different
optimizers~\citep{jiang2019improving}. 
Evaluating the combined effect of our
approach and adaptive optimizers remains an interesting open
direction. 
Additionally, we apply logit smoothing with weight = 0.9 to the loss function to mitigate against exploding gradients experienced often in training federated models.

For all algorithms, the following  hyperparameters are the same: client batch size=20, num clients per round=20, num rounds=1000, server learning rate=1.0, server momentum=0.9.
 For the remaining hyperparameters, we perform a sweep over parameters and use the eval dataset to choose the best. The best hyperparameters after sweeping are as follows.
\begin{itemize}
    \item \textsc{FedAvg}: client num epochs=1, client step size=0.05.
    
    \item \textsc{Agnostic}: client num epochs=1, client step size=0.05, domain learning rate=0.05.  NIST documentation shows that the EMNIST dataset comes from two writer sources: \textsc{Census} and \textsc{High School}~\citep{grother1995nist}. We use these two distinct sources as domains. 
    
    \item \textsc{HypCluster}: client num epochs=1, client step size=0.03, num clusters=2. We determined that 2 was the optimal number of clusters since for larger numbers of clusters, all clients essentially mapped to just 2 of them.
    
    \item \textsc{Finetune}: client num epochs=5, client step size=0.01. We use the best baseline model as the pre-trained starting global model and finetune for each client.
    
    \item \textsc{Dapper}: client num epochs=1, client step size=0.04. Similar to \textsc{Finetune}, we use the best baseline model as the pre-trained starting global model. For each client, we finetune the global model using a mixture of global and client data. Given $m_k$ client examples, we sample $5 \cdot m_k$ global examples.

    \item \textsc{Mapper}: client num epochs=1, client step size=0.03 . We use the same global hyperparameters as the starting global model with the above local hyperparameters. As stated previously, we use the same architecture for both local and global models and at each optimization step, we initialize the local model using the global parameters.
\end{itemize}

Over the course of running experiments, we observed a peculiar behavior regarding the original client ordering provided in the EMNIST dataset. If the seen and unseen split is performed on the original ordering of clients, the model performance between seen and unseen clients is very different. In particular, unseen is almost $10\%$ absolute worse than seen. Looking through the NIST documentation, we found that the data was sourced from two distinct sources: \textsc{Census} and \textsc{High School}~\citep{grother1995nist}. The original client ordering is derived from the data partition they are sourced from and \textsc{High School} clients were all in one split, resulting in the difference in model performance. Thus, we determined that shuffling the clients before splitting was better.
\end{document}